\documentclass[letterpaper, 10 pt, conference]{ieeeconf}
\IEEEoverridecommandlockouts
\usepackage[dvipsnames]{xcolor}
\usepackage{verbatim}
\usepackage{amstext}
\usepackage{amsmath}

\usepackage{amsthm}
\usepackage{amssymb}
\usepackage{graphicx}
\usepackage{caption}
\usepackage{rotating}
\usepackage{todonotes}
\usepackage{csvsimple}
\usepackage{booktabs}
\let\labelindent\relax % need this to avoid IEEE style conflicts
\usepackage[shortlabels]{enumitem}

\makeatletter
\let\NAT@parse\undefined
\makeatother
\usepackage{hyperref}

\graphicspath{ {./images/} }

\newtheorem{thm}{Theorem}[section]
\newtheorem{lemma}[thm]{Lemma}

\theoremstyle{definition}

\newtheorem{assump}{Assumption}

\newcommand{\norm}[1]{\Vert #1 \Vert}
\newcommand{\maxnorm}[1]{\lceil #1 \rceil}
\newcommand{\tr}{\intercal}
\newcommand{\diag}[1]{\mathrm{diag}\{#1\}}
\newcommand{\colVec}[1]{[#1]^\tr}
\newcommand{\real}{\mathbb{R}}
\newcommand{\nat}{\mathbb{N}}

\newcommand{\algName}{OneVision}
\newcommand{\controller}[0]{\pi}
\newcommand{\ccontroller}[0]{\controller^{c}}
\newcommand{\dcontroller}[0]{\controller^{d}}

\newcommand{\ideal}[1]{#1^*}
\newcommand{\pred}[2]{\tilde{#1}^{(#2)}}
\newcommand{\predi}[1]{\pred{#1}{i}} % forward-predicted by agent i
\newcommand{\preda}[1]{\pred{#1}{\cdot}}
\newcommand{\sysDy}{f}
\newcommand{\sysDyModel}{\hat{\sysDy}}
\newcommand{\obsDy}{h}
\newcommand{\obsDyModel}{\hat{\obsDy}}

\newcommand{\obsDelay}{T^x}
\newcommand{\actDelay}{T^u}
\newcommand{\comDelay}{T^c}
\newcommand{\totalDelay}{T^\text{all}}
\newcommand{\stateDelay}{\obsDelay}

\newcommand{\dLoss}{\ell}

\newcommand{\dRegLoss}{\dLoss_{\text{reg}}}

\newcommand{\error}[1]{\Delta #1}
\newcommand{\self}[1]{\bar{#1}}
\newcommand{\serror}[1]{\error{\self{#1}}}
\newcommand{\perror}[2]{\error{\pred{#1}{#2}}}
\newcommand{\perrori}[1]{\perror{#1}{i}}
\newcommand{\perrora}[1]{\perror{#1}{\cdot}}
\newcommand{\tauInit}{\tau_\text{init}}

\makeatletter
\csvset{
  autotabularcenter/.style={
    file=#1,
    after head=\csv@pretable\begin{tabular}{|*{\csv@columncount}{c|}}\csv@tablehead,
    table head=\hline\csvlinetotablerow\\\hline,
    late after line=\\,
    table foot=\\\hline,
    late after last line=\csv@tablefoot\end{tabular}\csv@posttable,
    command=\csvlinetotablerow},
  autobooktabularcenter/.style={
    file=#1,
    after head=\csv@pretable\begin{tabular}{*{\csv@columncount}{c}}\csv@tablehead,
    table head=\toprule\csvlinetotablerow\\\midrule,
    late after line=\\,
    table foot=\\\bottomrule,
    late after last line=\csv@tablefoot\end{tabular}\csv@posttable,
    command=\csvlinetotablerow},
}
\makeatother

\newcommand{\csvautobooktabularcenter}[2][]{\csvloop{autobooktabularcenter={#2},#1}}

\newcommand{\cl}{\,,\\}  %comma new line
\newcommand{\heading}[1]{\textbf{#1}\quad}

\title{\LARGE \bf
\algName{}: Centralized to Distributed Controller Synthesis\\ with Delay Compensation
}

\author{Jiayi Wei$^{1}$, Tongrui Li$^{1}$, Swarat Chaudhuri$^{1}$, Isil Dillig$^{1}$, and Joydeep Biswas$^{1}$% <-this % stops a space
\thanks{$^{1}$Computer Science Department,
        University of Texas at Austin, USA.
        {\tt\small \{jiayi, tongrui, swarat, isil, joydeepb\} @cs.utexas.edu}}%
}

\begin{document}

\maketitle
\thispagestyle{empty}
\pagestyle{empty}

\begin{abstract}
We propose a new algorithm to simplify the controller development for
distributed robotic systems subject to external observations, disturbances, and communication delays. Unlike prior approaches that propose specialized solutions to handling communication latency for specific robotic applications, our algorithm uses an arbitrary centralized controller as the specification and automatically generates distributed controllers with communication management and delay compensation. We formulate our goal as nonlinear optimal control---using a regret minimizing objective that measures how much the distributed agents behave differently from the delay-free centralized response---and solve for optimal actions w.r.t. local estimations of this objective using gradient-based optimization. We analyze our proposed algorithm's behavior under a linear time-invariant special case and prove that the closed-loop dynamics satisfy a form of input-to-state stability w.r.t. unexpected disturbances and observations. Our experimental results on both simulated and real-world robotic tasks demonstrate the practical usefulness of our approach and show significant improvement over several baseline approaches.
\end{abstract}

\section{Introduction}
% What is the problem?
We are interested in distributed multi-agent control of robots in environments with unknown conditions or obstacles. Examples of such settings include autonomous convoy driving following a human driver and autonomous formation control of a fleet that needs to change formations in response to obstacles. Unlike in applications with fixed formation~\cite{olfati-saber_distributed_2002, dunbar_distributed_2006} or trajectory control~\cite{ghommam2009coordinated}, the agents' behavior can vary significantly based on environmental observations, such as the observed trajectory of the lead car in the convoy setting or a narrow tunnel for the formation switching setting. Thus, it is preferable to specify the \emph{behavior} of the robotic fleet, rather than their \emph{execution}, as a desired ideal central controller. Unfortunately, such ideal central controllers cannot be executed directly in a distributed setting since each agent is only capable of observing their own local state, and communication latency leads to delayed observations of other agents. While there have been a few specialized solutions for handling communication latency for specific controllers such as formation control~\cite{liao_distributed_2017} and coordinated path following~\cite{ghabcheloo_coordinated_2009}, synthesizing distributed controllers from arbitrary central controllers while accounting for communication delays has remained an open problem until now.

In this paper, we present \algName{}, an algorithm for distributed control of multi-agent systems with local observation and disturbances, in the presence of communication delays. \algName{} accepts an ideal central control function for a multi-agent system as well as a system dynamics and observation model. Given the ideal central control function, \algName{} generates local plans at every time step by minimizing a \emph{regret loss} using gradient-based optimization. This regret loss is defined as the difference between the predicted future states and actions and an \emph{ideal fleet trajectory} computed by forward-predicting the central controller on delay-compensated local observations from all agents. Since the ideal fleet trajectory cannot be locally computed in real time due to communication delays, each synthesized distributed controller also computes a local approximation of the ideal fleet trajectory and plans its future actions against this approximated objective.

Although \algName{} works with arbitrary discrete-time multi-agent controllers, we limit our theoretical analysis to cases where the system dynamics and centralized controller are linear time-invariant. We prove that the distributed agents' execution generated by \algName{} converges to the ideal fleet trajectory and is stable in the sense that smaller external disturbances lead to staying closer to the ideal trajectory. In addition, we provide empirical evidence of convergence and stability for a number of non-linear examples.

% need robust, accurate execution. How are these controllers specified? 
% What are the challenges?

% What are others doing, and what is their limitation?

% Refer to trajectory following or formation, in the absence of obstacles, and local obstacle avoidance, while \algName{} can handle obstacles in the global form

% What alternative are we proposing?

% What new features does it enable?

% Summarize: contributions, theoretical and empirical results.

We summarize our contributions as follows:
\begin{itemize}
    \item We present \algName{}, a general algorithm for synthesizing distributed controllers from a centralized controller specification, under the presence of unknown local observations and disturbances. 
    \item We analyze the close-loop behavior of our algorithm in a linear time-invariant special case and provide theoretical guarantees on the resulting performance. Our analysis provides an error bound that is independent of the amount of the delays.
    \item We implement the proposed algorithm, %in Julia 
     experimentally evaluate our algorithm on 4 multi-agent tasks, and demonstrate the practical usefulness of our approach.
\end{itemize}

\section{Related Work} 

\paragraph{Synthesis Techniques for Multi-Robot Systems}
There has been a long line of work on synthesizing reactive controllers from temporal logic specifications for multi-robot systems~\cite{kloetzer2009automatic, kloetzer2011multi, moarref2017decentralized}. These approaches typically create a discrete abstraction of the system and synthesize hybrid controllers that fulfill the logical specifications. In contrast, we focus on synthesizing continuous distributed controllers from a centralized controller specification. 

\paragraph{Model Predictive Control (MPC)}
% Our work is closely related to model predictive control (MPC), where numerical optimization is applied at every time step to compute an optimal action that minimizes a receding-horizon loss, optionally subject to additional constraints~\cite{garcia1989model,noauthor_model_2009, rakovic_handbook_2019}. 

MPC has found use in several domains that are related to this work, including controlling distributed multi-agent systems~\cite{richards_decentralized_2004, richards_robust_2007} and distributed systems with time delays~\cite{li_distributed_2013, hahn_distributed_2018}. Recently, MPC has also been applied to robotics applications such as trajectory tracking~\cite{kamel_linear_2017}, vehicle control~\cite{cairano_stochastic_2014}, flight control~\cite{alexis_robust_2016}, and cooperative landing~\cite{persson_autonomous_2019}.
Unlike many prior approaches where MPC is used to directly optimize a global performance objective defined across multiple agents, we use MPC mainly as a local control planning strategy to reduce the discrepancy between each agent's own trajectory and the corresponding (centrally predicted) ideal fleet trajectory. 
% We believe the local nature of this approach can lead to better scalability and copes well with  information constraints imposed by the communication delays.

\paragraph{Distributed Control Designs and Applications}
In recent control and robotics literature, many specialized distributed controllers have been proposed for various application domains such as coordinated trajectory tracking and path following~\cite{aguiar2007coordinated, ghabcheloo_coordinated_2009, ghommam2009coordinated}, vehicle formation control~\cite{dunbar_distributed_2006, ren_distributed_2008, liao_distributed_2017}, traffic control~\cite{tettamanti2010distributed, street_multi-robot_nodate}, and information consensus~\cite{noauthor_information_2007, wen_consensus_2012}. In this work, we take a different approach and instead aim at reducing the future effort required to develop these distributed systems using a general controller synthesis framework applicable to different robotic applications. 

\paragraph{Centralized Formation Control}
Lastly, there are many prior works on multi-vehicle formation control using a centralized control law~\cite{leonard_virtual_2001, dang_formation_2019}; some also deal with the challenging case of nonholonomic robots~\cite{desai_modeling_2001, consolini_leaderfollower_2008}. In our 2D formation experiments, we employ a simple centralized control scheme based on reference point tracking~\cite{danwei_wang_full-state_2003} and rotational repulsive forces~\cite{dang_formation_2019}.

\section{Problem Definition}

Formally, our goal is to synthesize distributed robotic controllers from a given centralized controller, dynamics model, and observation model---subject to sensor noise, external disturbances and communication, actuation, and observation delays---such that the joint behaviors of the synthesized controllers approximate that of the centralized controller. For a setting with $N$ robots, the inputs to our problem are: 

\begin{enumerate}[(a)]
    \item $\sysDyModel_i$, the discrete-time dynamics models of robot $i$, for $i \in \{1\ldots N\}$, given in the form
    \begin{equation}
        \label{eq:open-dynamics}
        x_i(t+1) = \sysDyModel_i(x_i(t), u_i(t), t)\,,
    \end{equation}
    where $t\in \nat$ is the time index, $x_i(t) \in \real^{n^x_i}$ and $u_i(t) \in \real^{n^u_i}$ is robot $i$'s state and actuation vector at time step $t$, respectively. Note that $\sysDyModel_i$ can be different from the true dynamics $f_i$, with the difference being modeled as external disturbance. 
    \item $\obsDyModel_i$, the local observation model of robot $i$, for $i \in \{1\ldots N\}$, given in the form
    \begin{equation}
        \label{eq:obs-model}
        z_i(t+1) = \obsDyModel_i(z_i(t),t)\,,
    \end{equation}
    where $z_i \in \real^{n^z_i}$ is the observation of robot $i$.\footnote{Note that here we require observations to be defined as some state-independent quantities. For example, instead of using the distance to an obstacle as $z$ (which depends on the position of the robot) we can define $z$ as the obstacle's position (whose true value does not depend on where we perform the measurement).} Similarly, $\obsDyModel$ can be different from $h$, resulting in unexpected observations.
    \item $\ccontroller$, the centralized controller of the form
    \begin{equation}
    \label{eq:ccontroller}
        u(t) = \ccontroller(x(t), z(t), t)\,,
    \end{equation}
    where $x(t) = \colVec{x_1(t), \ldots, x_N(t)} \in \real^{N n^x}$ is the global state vector formed by vertically concatenate all robot states, and similarly, $z(t) = \colVec{z_1(t), \ldots, z_N(t)}$ $\in \real^{N n^z}$ and $u(t) = \colVec{u_1(t), \ldots, u_N(t)} \in \real^{N n^u}$.
    \item $\obsDelay, \actDelay, \comDelay \in \nat^+$, the discrete-time observation, actuation, and communication delay of this robotic fleet.
\end{enumerate} 

From the inputs given above, we want to synthesize $N$ distributed controllers $\dcontroller_{i},\ \forall i\in \{1\ldots N\}$ of the form
\begin{align}
    \label{eq:dcontroller}
    u_i(t+\actDelay) = \dcontroller_i(\mathcal{X}_i(t),\mathcal{Z}_i(t),\mathcal{U}_i(t),t)
\end{align}
where $\mathcal{X}_i(t),\mathcal{Z}_i(t),\mathcal{U}_i(t)$ denote the parts of the entire fleet's state history, observation history, and actuation history that are available to agent $i$ at time $t$, subject to constraints imposed by the delays. For example, we have $\mathcal{X}_i(t) = \{ x_i(\tau) | \tau \leq t - \obsDelay \} \cup \{ x_j(\tau) | j \neq i, \tau \leq t - \obsDelay - \comDelay \}$.

Since our goal is to make the distributed agents behave like they were controlled by the centralized controller $\ccontroller$, we need to formally define a loss that measures the distance between $\dcontroller$ and $\ccontroller$. In this work, we have considered two different ways to define such a loss.

\heading{Option 1: Action Loss} Since our goal is to make every agent to take an action that is similar to the one given by the centralized controller, one intuitive way to define such a loss is as
\begin{equation*}
    \ell_\text{act}(t) = \norm{\ccontroller(x(t), z(t), t) - u(t)}\,,
\end{equation*}
which simply measures the difference between the actuation output by the centralized controller (given the current state and observation) and the actual actuation $u(t)$ output by the distributed controllers. 

Minimizing this loss requires each agent to accurately predict the current state $x(t)$ and observation $z(t)$ of the entire fleet, such that we can define the output of $\dcontroller_i$ as $\ccontroller_i(\hat{x}^{(i)}(t), \hat{z}^{(i)}(t), t)$, where $\hat{x}^{(i)}(t) \approx x(t), \hat{z}^{(i)}(t) \approx z(t)$ are the $i$th agent's prediction of the current fleet state and observation.

However, predicting $\hat{x}$ and $\hat{z}$ in the closed-loop dynamics can lead to an infinite recursion. To see this, note that each agent's actuation depends on its prediction of other agent's states since
\[ 
    \forall i, \ u_i(t) = \ccontroller_i(\hat{x}^{(i)}(t), \hat{z}^{(i)}(t), t) \,.
\]
But for agent $i$ to predict agent $j$'s state, it will further need to predict $j$'s action at the previous time step $t'=t-1$:
\[
    \forall i,j, \ \hat{x}^{(i)}_j(t) = \sysDyModel_j(\hat{x}^{(i)}_j(t'), u^{(i)}_j(t'), t')\,.
\]
But this in turn requires predicting how agent $j$ would have predicted other agent's states:
\[
    \forall i,j,\  u_j^{(i)}(t') = \ccontroller_j(\hat{x}^{(i,j)}(t'), \hat{z}^{(i,j)}(t), t')\,,
\]
where the notation $\hat{x}^{(i,j)}(t')$ denotes agent $i$'s prediction (made at $t$) of agent $j$'s prediction (made at $t'$) of the fleet state at time $t'$. Therefore, we can keep unrolling the right hand side, resulting in an infinite recursion.

\heading{Option 2: Regret Loss} In this formulation, instead of requiring each agent to predict the current states of other agent, we introduce the notion of an \emph{ideal fleet trajectory} that---although not locally computable by each agent in real time---will always become computable later once more information become available through communication. We then define the loss as the difference between the actual fleet trajectory $(x, u)$ and this ideal fleet trajectory $(\ideal{x}, \ideal{u})$:
\begin{equation}
    \label{eq:regret-loss}
    \dRegLoss(t) = \norm{x(t)-\ideal{x}(t)}_{Q_x} + \norm{u(t)-\ideal{u}(t)}_{Q_u}
\end{equation}
for some positive definite $Q_u$ and positive semi-definite $Q_x$, with the notation $
\norm{a}_B = a^\tr B a$ denoting the quadratic norm of $a$ defined by matrix $B$. 

Assuming the true observation dynamics $\obsDy(z,t) = \obsDyModel(z,t) + \delta z(t)$, and the true system dynamics $\sysDy(x,u,t) = \sysDyModel(x,u,t) + \delta x(t)$, where $\delta z$ and $\delta x$ are the observation and state disturbance (both unknown a priori to us), we define the ideal fleet trajectory $\ideal{x}$ and $\ideal{u}$ as the solution to the closed-loop dynamics obtained by combining equation~\eqref{eq:open-dynamics}--\eqref{eq:ccontroller}, with $\sysDyModel$ and $\obsDyModel$ replaced by $\sysDy$ and $\obsDy$:
\begin{equation}
\label{eq:ideal-dynamics}
\begin{split}
    \ideal{x}_i(t+1) &= \sysDy_i(\ideal{x}_i(t), \ideal{u}_i(t), t) \cl
    \ideal{u}(t) &= \ccontroller(\ideal{x}(t), \ideal{z}(t), t) \cl
    \ideal{z}_i(t+1) &= \obsDy_i(\ideal{z}_i(t),t)\,.
\end{split}
\end{equation}
Note that since $\delta x_i$ and $\delta z_i$ can be measured at time $t+1$ from the observed state and actuation (assuming no observation noise) using eq.~\eqref{eq:open-dynamics} and ~\eqref{eq:obs-model} as 
\begin{align*}
    \delta x_i(t) &= x_i(t+1) - \sysDyModel_i(x_i(t), u_i(t), t)) \cl
    \delta z_i(t) &= z_i(t+1) - \obsDyModel_i(z_i(t), t)\,,
\end{align*} 
each agent will eventually be able to locally compute the same ideal fleet trajectory once other agents' $\delta x$ and $\delta z$ become available through communication. However, because communication takes time, the most recent part of the ideal fleet trajectory will have to be predicted initially and revised later, hence loss~\eqref{eq:regret-loss} in general cannot be zero and is caused by the ``regret" of each agent's past predictions.

Since minimizing this loss only requires each agent to predict the ideal fleet trajectory, which is considerably easier than predicting the actual fleet trajectory, we use this second loss definition in this work.

\begin{figure}
    \includegraphics[width=0.9\linewidth]{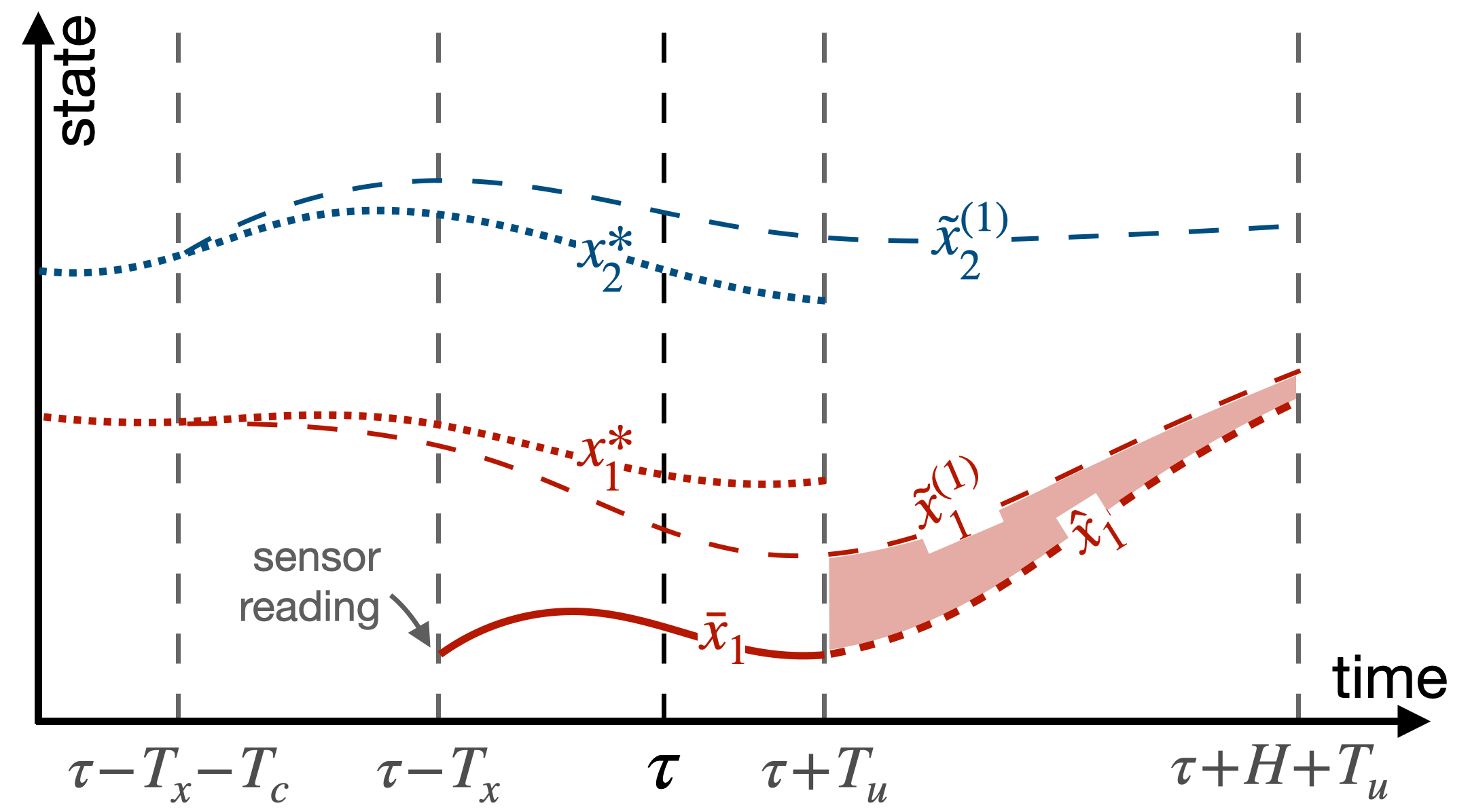}
    \centering
    \caption{
        \label{fig:algorithm-illustration}
        \textbf{How \algName{} works from the perspective of agent 1.} On the time axis, $\stateDelay, \actDelay$, and $\comDelay$ denote state, actuation, and communication delay, respectively. $\tau$ is the current time step, and $H$ is the prediction horizon. 
        The two dotted lines, $x^*_1$ and $x^*_2$, represent the ideal fleet trajectories, while $\pred{x}{1}_1$ and $\pred{x}{1}_2$ represent the most probable ideal fleet trajectory predicted by agent 1. The solid line $\bar{x}_1$ represents the self estimation computed using agent 1's recorded action history. Line $\hat{x}_1$ represents a planned future trajectory obtained by minimizing the discrepancy between $\hat{x}_1$ and $\pred{x}{1}_2$ (shown as the colored area.)}
    \vspace{-0.5cm}
\end{figure}

\section{\algName{} Overview}

At a high level, our algorithm makes control decisions based on three main steps: \emph{forward prediction}, \emph{self state estimation}, and \emph{local planning}.
In forward prediction, each agent tries to locally compute an \emph{estimated ideal fleet trajectory} that approximates the ideal fleet trajectory using all the information currently available to itself. As we will show in Section~\ref{sec:theorectical-analysis}, an important property of this step is that despite agents only having access to limited state information about other agents, the differences between these predicted fleet trajectories computed by different agents do not accumulate over time. In self state estimation, each agent uses its locally recorded actuation history and observed past state to predict its future state at $\tau+\actDelay$. Lastly, in local planning, each agent tries to plan its next action using model predictive control, with the goal of minimizing the discrepancy between its predicted future trajectory and the corresponding part in the estimated ideal fleet trajectory. These three steps are illustrated in Figure~\ref{fig:algorithm-illustration}.

Let $\tau$ be the current time step, we now describe each of these three steps in more details:

\textbf{1. Forward Prediction.} Every robot $i$ uses the newest information available to itself (as defined in eq.~\eqref{eq:dcontroller}) to forward-predict the most probable future ideal fleet trajectory $\predi{x}$ by solving the following initial value problem for the time span $\tau-\obsDelay-\comDelay-1 \leq t < \tau+\actDelay+H$, where $H$ is the prediction horizon:
\begin{align}
    \label{eq:forward-prediction}
    \begin{split}
    \predi{x}_j(t+1) &= \sysDyModel_j(\predi{x}_j(t), \predi{u}_j(t), t) + \delta \predi{x}_j \cl
    \predi{u}(t) &= \ccontroller(\predi{x}(t), \predi{z}(t), t) \cl
    \predi{z}_j(t+1) &= \obsDyModel_j(\predi{x}_j(t), \predi{z}_j(t), t) + \delta \predi{z}_j\,,
    \end{split}
\end{align}
for $j\in \{1 \ldots N\}$. Let $\tauInit = \tau-\obsDelay-\comDelay-1$, the above problem is subject to the initial conditions
\begin{equation}
    \label{eq:forward-initial}
    \predi{x}_j(\tauInit|\tau) = \predi{x}_j(\tauInit|\tau-1)\,,
\end{equation}
where the notation $\predi{f}_j(t|\tau)$ means ``the prediction of $\tilde{f}_j(t)$ made by robot $i$ at time $\tau$''. The disturbance terms $\delta \predi{x}$ and $\delta \predi{z}$ are defined to be zero unless the corresponding trajectory information is available to robot $i$, i.e.,
\begin{equation}
    \label{eq:forward-constraints}
    \delta \predi{x}_j, \delta \predi{z}_j = \begin{cases}
        \delta x_j, \delta z_j  & j=i \text{ and } \tauInit \leq t < \tau - \obsDelay \\
        & \text{or } j\neq i \text{ and } t = \tauInit \cl
        0, 0 & \text{otherwise.}
    \end{cases}
\end{equation} 

\textbf{2. Self State Estimation.} Every robot $i$ then estimates its actual state at time $\tau+\actDelay$ using its actuation history $u_i$ by solving the following initial value problem for $\tau-\obsDelay \leq t < \tau + \actDelay$:
\begin{equation}
    \label{eq:self-state-estimation}
    \self{x}_i(t+1) = \sysDyModel(\self{x}_i(t), u_i(t), t)
\end{equation}
subject to the initial condition
\begin{equation}
    \label{eq:self-state-init}
    \self{x}_i(\tau-\obsDelay) = x_i(\tau-\obsDelay)\,.
\end{equation}
This gives the self state estimation $\self{x}_i(\tau+\actDelay)$, which will be used in the next step.

\textbf{3. Local Planning.} Every robot $i$ then uses the predicted $\predi{x}$ and $\predi{u}$ from step 1 as the \emph{most probable approximation} to the ideal fleet trajectory $\ideal{x}$ and $\ideal{u}$ and tries to locally minimize its future regret by solving the following optimization problem for the control time span $\tau + \actDelay \leq t < \tau + \actDelay + H$:
\begin{equation}
    \label{eq:local-planning-obj}
    \operatorname*{Minimize}_{\{\hat{u}_i(t)| t\}} \ \sum_{t=\tau + \actDelay}^{\tau+\actDelay+H-1} \dLoss_i(t)\,,
\end{equation} where
\begin{equation*}
    \dLoss_i(t) = \norm{\hat{x}_i(t)-\predi{x}_i(t)}_{Q_x} + \norm{\hat{u}_i(t)-\predi{u}_i(t)}_{Q_u}\,,
\end{equation*}
subject to the initial condition
\begin{equation}
    \label{eq:local-planning-initial}
    \hat{x}_i(t) = \self{x}_i(t)\,,\ \text{for } t = \tau+\actDelay
\end{equation}
and the dynamics constraints
\begin{equation}
    \label{eq:local-planning-dynamics}
    \hat{x}_i(t+1) = \sysDyModel(\hat{x}_i(t), \hat{u}_i(t), t)\,,
\end{equation}
for $\tau + \actDelay \leq t < \tau+\actDelay+H$.
Robot $i$ then takes the first actuation from the optimal solution as its next actuation, i.e., we have
$ u_i(\tau+\actDelay) = \hat{u}_i(\tau+\actDelay|\tau)\,.$

\textbf{Initialization}
\algName{} needs an initial history of $x$, $z$, and $u$ to start with. As we will see from the analysis in the next section, our algorithm is not sensitive to the errors introduced during initialization as long as the initialization guarantees that all robots make the same forward prediction $\predi{x}$ initially. In our implementation, we assume the initial condition of the entire fleet is available to all robots during initialization, and we simply initialize the history trajectories as constant functions whose value equals the corresponding initial condition.

\section{Theoretical Analysis}
\label{sec:theorectical-analysis}
In this section, we analyze our algorithm's closed-loop behavior in the special case of linear time-invariant (LTI) dynamics and controller. Our main goal is to answer the following two questions: 1) When our model error $\delta x$ and $\delta z$ is zero, how fast does the true fleet trajectory $x$ converge to the ideal fleet trajectory $\ideal{x}$? 2) When $\delta x$ and $\delta z$ is small, does $x$ stays close to $\ideal{x}$? We will answer these questions in Proposition~\ref{thm:main-result} at the end of this section, stated in the form of input-to-state stability. But before we can prove Proposition~\ref{thm:main-result}, we first prove a few lemmas and make additional assumptions.

\begin{lemma}[Inital condition of forward prediction]
    \label{thm:forward-initial}
    Initial condition \eqref{eq:forward-initial} in the forward prediction step initializes $\predi{x}$ to the ideal fleet state $\ideal{x}$ (as shown in Figure~\ref{fig:algorithm-illustration}). i.e., 
    \[
        \predi{x}_j(\tauInit|\tau) = \ideal{x}_j(\tauInit),\ \forall j,\ \forall \tau \geq 0\,.
    \]
\end{lemma}

\begin{proof}
    We prove this by induction. The base case at $\tau = 0$, $\predi{x}_j(\tauInit|\tau) = \ideal{x}_j(\tauInit)$ holds because of the initialization step. For the inductive case, assume $\predi{x}_j(\tauInit|\tau) = \ideal{x}_j(\tauInit)$ at time $\tau$, we aim to show that $\predi{x}_j(\tauInit+1|\tau+1) = \ideal{x}_j(\tauInit+1)$. Compare the dynamics \eqref{eq:ideal-dynamics} of $\ideal{x}$  with the dynamics \eqref{eq:forward-prediction} of $\predi{x}$ , we see that they become identical if $\delta \predi{x}(t) = \delta x(t)$ and $\delta \predi{z}(t) = \delta z(t)$. And from the history constraints \eqref{eq:forward-constraints}, we see that this is indeed the case at $t=\tauInit$; hence, we have $\predi{x}_j(\tauInit+1|\tau) = \ideal{x}_j(\tauInit+1)$. Finally, apply the initial condition~\eqref{eq:forward-initial}, we have $\predi{x}_j(\tauInit+1|\tau+1) = \ideal{x}_j(\tauInit+1)$\,.
\end{proof}

Next, we introduce the following LTI assumptions about the system dynamics and controllers.\footnote{For mildly nonlinear systems, the following assumptions can hold temporarily by linearizing the system behavior around the current operating point. }
\begin{assump}
    \label{assump:x-dynamics}
    For $t \geq 0$, all robots have the LTI dynamics of the form
    \begin{align}
        \label{eq:sys_dynamics-linear}
        \hat{\sysDy}_i(x_i, u_i, t) &= A_i x_i + B_i u_i + \hat{w}_i(t)\cl
        \sysDy_i(x_i, u_i, t) &= A_i x_i + B_i u_i + \ideal{w}_i(t)\,,
    \end{align}
    in which $A_i$ and $B_i$ are constant matrices, and the time-dependent terms satisfy $\ideal{w}_i(t) - \hat{w}_i(t) = \delta x(t)$. We also assume that norm $\vert \lambda_k \vert \leq 1$ for all eigenvalues $\lambda_k$ of $A_i$, i.e., the system dynamics is not exponentially unstable.
\end{assump}

\begin{assump}
    \label{assump:z-dynamics}
    For $t \geq 0$, the observation dynamics have LTI dynamics of the form
    \begin{align}
        \label{eq:obs_dynamics-linear}
        \obsDyModel(z_i, t) = C_i z_i + \hat{\mu}_i(t)\cl
        \obsDy(z_i, t) = C_i z_i + \ideal{\mu}_i(t)\,,
    \end{align}
    in which $C_i$ has all its eigenvalues $\lambda_k$'s norm $\vert \lambda_k \vert \leq 1$, and $\ideal{\mu}(t) - \hat{\mu}_i(t) = \delta z(t)$.
\end{assump}

\begin{assump}
    \label{assump:controller}
    For $t\geq 0$, the centralized controller $\ccontroller$ is also LTI w.r.t. $x$ and $z$  and has the form
    \begin{equation*}
        \ccontroller(x,z,t) = -K_x x + K_z z + v(t)\,,
    \end{equation*}
    and $K_x$ stabilizes the closed-loop dynamics, i.e., $\vert \lambda_k \vert < 1$ for all eigenvalues $\lambda_k$ of the matrix $A-B K_x$. Here, $A=\diag{A_i |\, \forall i}$ and $B=\diag{B_i|\, \forall i}$ are the block-diagonal matrices of the overall system.
\end{assump}

We next prove the following error bounds of the forward prediction and self state estimation step.

\begin{lemma}[Self estimation error]
    \label{thm:self-estimation}
    Let $\serror{x}_i = \self{x}_i - x_i$, $\error{\hat{w}}_i = \hat{w}_i - \ideal{w}_i$. At any given time $\tau \geq 0$, we have 
    \begin{equation*}
    \norm{\serror{x}_i(\tau+\actDelay)} \leq \beta_i(\stateDelay+\actDelay)\, \maxnorm{\error{\hat{w}}_i}\,,
    \end{equation*}    
    where $\beta_i$ is a positive definite polynomial of order $m_i$\footnote{when $A_i$ is stable, $\beta_i$ reduces to a constant.}, $m_i$ depends on the number of unit-norm eigenvalues of $A_i$, and 
    \[
        \maxnorm{\error{\hat{w}}_i} = \max_{-\stateDelay \leq t \leq \actDelay}{\norm{\error{\hat{w}}_i(\tau+t)}}
    \] is the maximal norm of $\error{\hat{w}}_i$ between $\tau-\stateDelay$ and $\tau+\actDelay$\,.
\end{lemma}
\begin{proof}
    Using \eqref{eq:open-dynamics}, \eqref{eq:self-state-estimation}, we can write down the error dynamics as
    \begin{equation}
        \label{eq:self-error-dynamics}
        \serror{x}_i(t+1) = A_i \serror{x}_i(t) + \error{\hat{w}}_i(t)
    \end{equation}
    with the initial condition (due to \eqref{eq:self-state-init})
    \begin{equation*}
        \serror{x}_i(\tau-\obsDelay) = 0\,.
    \end{equation*}
    Solving this linear system gives us the value of $\serror{x}$ at $\tau + \actDelay$:
    \begin{equation*}
        \serror{x}_i(\tau+\actDelay) = \sum_{t=1}^{\stateDelay+\actDelay}{(A_i)^t \error{\hat{w}}_i(\tau+\actDelay-t)}\,.
    \end{equation*}
    Since $A_i$ contains no eigenvalues whose norm is greater than 1, (Assumption~\ref{assump:x-dynamics}), we can bound $\norm{(A_i)^t}$ by a polynomial of order $m_i'$ on $t$, where $m_i'$ is the number of unit-norm eigenvalues of $A_i$ whose algebraic multiplicity $\neq$ geometric multiplicity. Also bound $\norm{\error{\hat{w}}_i(t)}$ by $\maxnorm{\error{\hat{w}}_i}$, we arrive at the conclusion by setting $m_i=m_i'+1$.
\end{proof}

Similarly, we now provide a bound for the prediction error $\perrori{x}(t|\tau) = \predi{x}(t|\tau) - \ideal{x}(t)$ and $\perrori{z}(t|\tau) = \predi{z}(t|\tau) - \ideal{z}(t)$\,. 
\begin{lemma}[Forward prediction error]
    \label{thm:forward-error}
    At any time $\tau \geq 0$, we have
    \begin{align}
        \label{eq:forward-z-bound}
        \norm{\perrori{z}(\tau+\actDelay|\tau)}
        & \leq \gamma_i(\totalDelay) \maxnorm{\error{\hat{\mu}}}\,, \\
        \label{eq:forward-x-bound}
        \norm{\perrori{x}(\tau+\actDelay|\tau)}
        & \leq a_i \maxnorm{\error{\hat{\mu}}} + b \maxnorm{\error{\hat{w}}}\,.
    \end{align}
    where $\totalDelay = \obsDelay + \actDelay + \comDelay$. $\gamma_i$ is a polynomial of order $n_i$ that depends on the eigenvalues of $C_i$'s, and $a_i$, $b$ are constants that depends on $A,B,K_x,K_z,\text{and }C$. $\error{\hat{\mu}} = \hat{\mu} - \ideal{\mu}$ is the prediction model error, and $\maxnorm{\error{\hat{\mu}}}$ and $\maxnorm{\error{\hat{w}}}$ are the corresponding maximal error norms between $\tau-\obsDelay-\comDelay$ and $\tau + \actDelay$.
\end{lemma}

\begin{proof}
    Take the difference of the dynamics of $\predi{z}_i$ and $\ideal{z}_i$ using \eqref{eq:obs_dynamics-linear}, we can write down the error dynamics of $\error{\predi{z}_i}(t)$ as
    \begin{equation*}
        \error{\predi{z}_j}(t+1) = C_j \error{\predi{z}_j}(t) + \error{\mu_j}(t)
    \end{equation*}
    with the initial condition (which can be obtained from the history constraints \eqref{eq:forward-constraints})
    \[
        \begin{cases}
            \error{\predi{z}_i}(\tauInit+\obsDelay) = 0\,,  & j = i\,\\
            \error{\predi{z}_i}(\tauInit) = 0\,,   & j\neq i\,.
        \end{cases}
    \]
    And since $C_i$ has no eigenvalues whose norm is greater than 1, similar to the argument in the proof of Lemma~\ref{thm:self-estimation}, $\error{\predi{z}_i}$ only grows at most at polynomial speed, hence we have
    \begin{equation*}
        \norm{\perrori{z}(\tau+\actDelay)}
        \leq \gamma_i'(\totalDelay-\stateDelay) \maxnorm{\error{\hat{\mu}}_i} + \sum_{j \neq i} \gamma_j'(\totalDelay) \maxnorm{\error{\hat{\mu}}_j}\,.
    \end{equation*}
    We can then bound the right-hand-side with $\gamma_i(\totalDelay) \maxnorm{\error{\hat{\mu}}}$ and arrive at \eqref{eq:forward-z-bound}.
    
    Now using the linear form of $\ccontroller$ given by Assumption~\ref{assump:controller}, we can also write down the dynamics of $\perrori{x}$ as
    \begin{equation*}
    \begin{split}
    \perrori{x}(t+1) = &\, (A - B K_x) \perrori{x}(t) \\
     & + B K_z \error{\predi{z}}(t) + \error{\predi{w}_j}(t)
    \end{split}
    \end{equation*}
    with the initial condition (which holds by Lemma~\ref{thm:forward-initial})
    \[ 
        \perrori{x}(\tauInit) = 0\,.
    \]
   Since $\vert \lambda_k \vert < 1$ for all eigenvalues $\lambda_k$ of $A-B K_x$ (Assumption~\ref{assump:controller}), and $\error{\predi{z}}$ grows at most at polynomial speed, we can bound $\perrori{x}(\tau + \actDelay | \tau)$ using \eqref{eq:forward-x-bound}.
\end{proof}

To simplify the analysis of local planning, we also assume that the prediction horizon $H$ is very long such that the resulting optimal action is linear feedback.

\begin{assump}
    The prediction horizon $H \rightarrow \infty$. 
\end{assump}

\begin{lemma}[Local planning provides linear feedback]
    \label{thm:local-planning}
    The optimal actuation given by the local planning step has the form
    \begin{equation*}
        u_i(t) = \predi{u}_i(t|t-\actDelay) - K^L_i (\self{x}_i(t) - \predi{x}_i(t|t-\actDelay))\,,
    \end{equation*} 
    where $K^L_i$ is some constant matrix that stabilizes the system.
\end{lemma}

\begin{proof}
    Take the difference between \eqref{eq:local-planning-dynamics} and \eqref{eq:forward-prediction}, and notice that there are no history constraints during $t\geq \tau+\actDelay$, we obtain
    \begin{align*}
    \hat{x}_i(t+1) - \predi{x}_i(t+1) = &\ A_i\,(\hat{x}_i(t) - \predi{x}_i(t|\tau)) \\
    & + B_i\,(\hat{u}_i(t)-\predi{u}_i(t|\tau))
    \end{align*}  
    for $t \geq \tau + \actDelay$, which---when combined with the objective \eqref{eq:local-planning-obj} and the assumption $H\approx \infty\,$---matches the form of a linear quadratic regulator (LQR) problem. Hence, the optimal solution is a linear feedback law given by
    \begin{equation*}
        \hat{u}_i(t) - \predi{u}_i(t|\tau) = - K^L_i (\hat{x}_i(t|\tau) - \predi{x}_i(t|\tau))\,,
    \end{equation*}
    where the gain $K^L_i$ stabilizes the system and can be obtained by solving the discrete-time algebraic Riccati equation~\cite{lancaster1995algebraic}. 

    Take $t=\tau+\actDelay$ in the above, and notice that $\hat{x}_i(\tau+\actDelay|\tau) = \self{x}_i(\tau+\actDelay)$ (equation~\eqref{eq:local-planning-initial}), we arrive at the conclusion.
\end{proof}

We are now ready to prove our main result. 
\begin{thm}[closed-loop stability]
    \label{thm:main-result}
    Let $\error{x} = x - \ideal{x}$ be the distance between the actual fleet trajectory and the ideal fleet trajectory, we have the following bound
    \begin{equation}
        \label{eq:closed-loop-bound}
        \norm{\error{x}(t)} \leq c_1 e^{-\lambda t} \norm{\error{x}(0)} + d_1 \maxnorm{\error{\hat{w}}} + d_2 \maxnorm{\error{\hat{\mu}}}\,,
    \end{equation}
    where $c_1, \lambda, d_1, \text{and } d_2$ are all constants independent of the delays, and the maximal norm $\maxnorm{\cdot}$ is defined on the interval $0 \leq t < \infty$.
\end{thm}

\begin{proof}
    We have the actual closed-loop dynamics
    \begin{equation*}
    x(t+1) = A x(t) + B u(t) + \ideal{w}(t)\,,
    \end{equation*}
    where $u(t)$ is given by Lemma~\ref{thm:local-planning}.
    We also have the ideal dynamics
    \begin{equation*}
    \ideal{x}(t+1) = A \ideal{x}(t) + B \ideal{u}(t) + \ideal{w}(t)\,.
    \end{equation*}
    Take the difference, we obtain
    \begin{align*}
    \error{x}(t+1) & = A \error{x}(t) + B \error{u}(t)\,.
    \end{align*}
    Expand $u$ and $\ideal{u}$'s definitions, and use the notation $\preda{f}$ to denote $\colVec{\pred{f}{j}_j |\ \forall j}$, we have
    \begin{align*}
    \error{u} & = \preda{u} - \ideal{u} - K^L(\self{x}(t) - \preda{x}) \\
    & = \colVec{(K_z \pred{z}{j})_j |\ \forall j} - K_z \ideal{z} - K_x (\preda{x} - \ideal{x}) \\
    & \quad\ - K^L (x - \ideal{x} + \self{x} - x + \ideal{x} - \preda{x})\\
    & = \colVec{(K_z \perror{z}{j})_j|\ \forall j} - K_x \perrora{x}\\ 
    & \quad\ - K^L (\error{x} + \serror{x} - \perrora{x})\,,
    \end{align*}
    in which the time arguments are omitted for brevity.
    Substitute the above into the previous equation, we obtain
    \begin{equation}
        \label{eq:closed-loop-error-dynamics}
        \error{x}(t+1) = (A - B K^L) \error{x}(t) + B \epsilon(t)\,,
    \end{equation}
    where the disturbance term
    \begin{equation*}
        \epsilon(t) = \colVec{(K_z \perror{z}{j})_j|\ \forall j} - (K_x - K^L) \perrora{x} - K^L \serror{x} 
    \end{equation*}
    is clearly bounded by the bound of $\perror{z}{j}$, $\perrora{x}$, and $\serror{x}$ (given by Lemma~\ref{thm:forward-error} and \ref{thm:self-estimation}). Thus, we can write the bound as
    \begin{equation*}
        \norm{\epsilon(t)} \leq \alpha_1(t) \maxnorm{\error{\hat{\mu}}} + \alpha_2(t) \maxnorm{\error{\hat{w}}}
    \end{equation*}
    using some polynomial $\alpha_1$ and $\alpha_2$ (whose concrete forms are not needed for this proof).

    Since $A-BK^L$ has only eigenvalues with norms $< 1$, and $\norm{\epsilon(t)}$ grows at most at polynomial rate, we can bound the solution to \eqref{eq:closed-loop-error-dynamics} with \eqref{eq:closed-loop-bound}.
\end{proof}

\section{Experimental Evaluation}

\begin{figure}
    \includegraphics[width=0.9\linewidth]{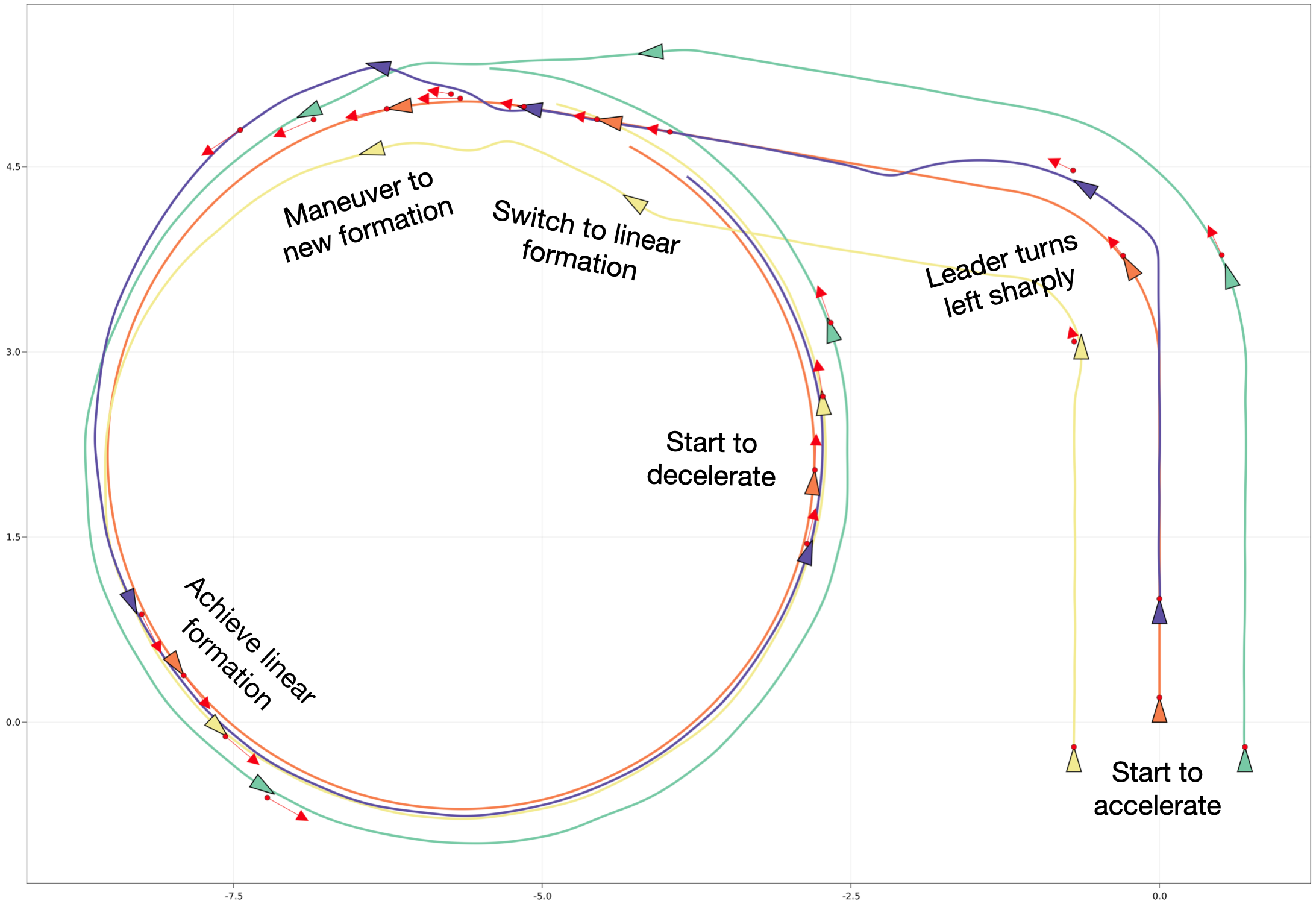}
    \centering
    \caption{
        \label{fig:formation-switching}
        \textbf{Simulation of 2D Formation Switching running \algName{}}. The orange car is the formation leader and is controlled by external inputs, while the three other cars try to follow the leader. Red dots show formation reference points.}
    \vspace{-0.3cm}
\end{figure}

\heading{Implementation} \label{sec:impl}
Our implementation of \algName{} consists of around 2000 lines of Julia code and allows the user to write centralized controllers as ordinary Julia functions. To minimize the regret loss \eqref{eq:regret-loss}, we use the L-BFGS optimizer~\cite{byrd1995limited} implemented by the Optim package~\cite{mogensen2018optim}, which employs automatic differentiation to compute gradient information needed by the optimization process. To ensure convergence to the same local optimum between time steps when dealing with nonlinear dynamics, we always feed the solution found in the previous time step as the initial solution to the optimizer in the next time step.

\heading{Simulation Tasks}
We use the following 4 simulation tasks to compare \algName{} with several other baseline controllers. We simulate 20 seconds for each task. 

\begin{enumerate}
    \item \textbf{1D Leader Linear:} 1-dimensional leader-follower driving task with a linear (PID) controller. The leader's goal is to reach a desired velocity, while the follower tries to stay close to the leader
    \item \textbf{1D Leader With Obstacle:}\label{sec:1dobs} 1-dimensional leader-follower task with an obstacle of unknown position. The leader can observe the obstacle when it is within its sensor's range and will brake once the obstacle becomes too close. To control velocity, both robots use bang-bang control instead of linear control.
    \item \textbf{2D Formation Driving:}\label{sec:2dformation} 4 car-like robots driving on a 2-dimensional plane. The leader car is assumed to be controlled by external commands (in the form of external observations), while the other four cars follow the leader and try to maintain a circular formation while avoiding collision with each other.
    \item \textbf{2D Formation Switching:} Similar to the task above, but the leader also controls the formation of the fleet. At some point, the leader will switch the formation from a triangle to a straight line. A simulation result of this task is shown in Figure~\ref{fig:formation-switching}. 
\end{enumerate}

% What baselines do we compare against?
\heading{Baseline Controller Frameworks}
We compare \algName{} with 3 other baseline controller frameworks: 
\begin{itemize}
    \item \textbf{Naive:} This is the simplest controller framework that runs the centralized controller without performing any delay compensation. In other words, each agent simply treats their own observations as well as the information broadcasted by other agents as the current state of the fleet.
    \item \textbf{Local:} Under this controller framework, delay compensation is limited to local information. Each agent uses its local actuation history and dynamics model to predict away its state and observation delays.
    \item \textbf{ConstU:} This controller framework employs simple heuristics to perform global compensation by assuming other agent's actuation remains constant during forward prediction. It also performs local compensation using the same strategy as Local.
\end{itemize}

% What metrics do we measure?
\heading{Task-Sepcific Metrics}
To quantitively measure the performance on each task, in addition to the loss defined in \eqref{eq:regret-loss}, which is motivated by our proposed algorithm, we also define the following two method-independent, task-specific metrics. 

\begin{itemize}
    \item \textbf{Average Distance} (for task 1 and 2): defined as the time-averaged distance between the leader and follower in L2 norm, given by $\sqrt{\frac{1}{T}\int_0^T (p_1-p_2)^2 dt}$, where $p_1$ and $p_2$ are the positions of each car, and $T$ is the length of the simulation.
    \item \textbf{Average Deviation} (for task 3 and 4): defined as the time-averaged distance between each follower and their supposed position in the fleet formation, given by $\sqrt{\frac{1}{T(N-1)}
    \int_0^T \sum_{i=2}^{N}{\norm{p_i- \hat{p}_i}^2} dt}$, where $\hat{p}_i$ is the supposed position of follower $i$.
\end{itemize}

\heading{Default Parameters}
Unless stated otherwise, we use the following set of parameters across different tasks: We run all controllers at 20Hz, with communication delay $\comDelay=$ 50ms, observation delay $\obsDelay=$ 30ms, and actuation delay $\actDelay=$ 40ms\footnote{In our discrete formulation time is discrete. Hence, to handle non-integral delays like 30ms (which is less than 1 time step under 20Hz control), we actually run our framework at 100Hz but only replan actuation at every 5th time step.}. For sensor noise, we add Gaussian noise to the state vector at every time step, and for external disturbance, we also add Gaussian noise but it is limited only to velocity and steering angle. In both cases, the default noise strength is 0.005 (in SI units) at 100Hz. We set \algName{}'s prediction horizon to be $H = 20$, which corresponds to a time span of 1 second. By default, we also assume the dynamics models accurately match the true dynamics (excluding noise).

\begin{table}
    \centering
    \vspace{0.1cm}
    \caption{\label{tab:performance-loss}
        Simulation Performance (log loss)}
    {\small
    \csvautobooktabularcenter{"data/log_loss.csv"}
    }
\end{table}

\begin{table}
    \centering
    \caption{\label{tab:performance-metric}
        Simulation Performance (task-specific metrics)}
    {\small
    \csvautobooktabularcenter{"data/custom_metrics.csv"}
    }
    \vspace{-0.5cm}
\end{table}

\subsection{Performance under Default Parameters}

We run \algName{} along with the 3 other baselines under the default parameters and compare their performance in Table~\ref{tab:performance-loss} (regret loss) and Table~\ref{tab:performance-metric} (task-specific metrics). Each datum is obtained by averaging 100 random runs. \algName{} achieves consistently the best performance in all tasks under both performance measurements. 

\subsection{Sensitivity Analysis}

\begin{figure}
    \includegraphics[width=1\linewidth]{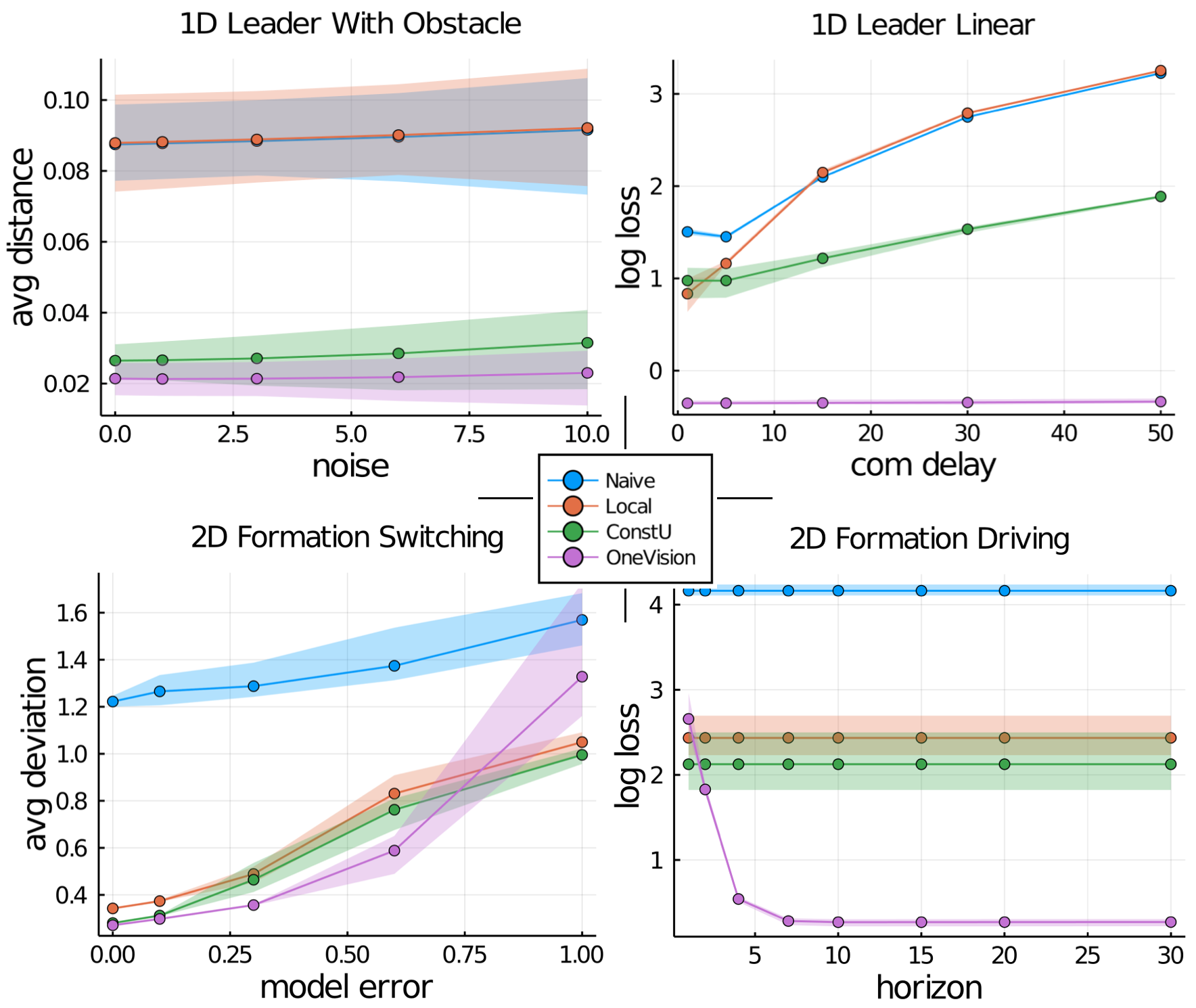}
    \centering
    \caption{
        \label{fig:sensitivity-analysis}
        Representative examples showing impact of sensor noise (upper left), communication delay (upper right), model inaccuracy (lower left), and prediction horizon (lower right).  }
    \vspace{-0.1cm}
\end{figure}

% Show the hardest tasks. Describe the general trends. Leave full results in appendix.
In this section, we modify the default parameters to study the impact of sensor noise, disturbance, delays, model inaccuracy, and prediction horizon. We vary each of these variables and compare the performance of different controller frameworks on all 4 simulation tasks (all results are measured using 10 random runs). Due to space constraints, we only present a few representative examples in this section and describe general trends we observed for the rest cases.\footnote{ Our full results available at \texttt{https://git.io/JqJ2x}}

\heading{Sensor Noise (Figure~\ref{fig:sensitivity-analysis} upper left)} We modify the sensor noise strength from 0 the 10 times the default strength and found that \algName{} still achieves the best performance across this range. We also observe similar trends when varying the amount of communication delay from 10ms to 500ms, both measured in log loss and in average distance/deviation. Particularly, for the 1D Leader Linear task, we notice that \algName{}'s performance is almost unaffected by the amount of the delay (Figure~\ref{fig:sensitivity-analysis} upper right), confirming the results given by Theorem~\ref{thm:main-result}. 

\heading{Model Error (Figure~\ref{fig:sensitivity-analysis} lower left))} To study the sensitivity w.r.t. model error, we define model error for task 1 and 2 as $r_1-1$, where $r_1$ is the ratio between the modeled car acceleration and the actual acceleration, and for task 3 and 4 as $r_2-1$, where $r_2$ is the ratio between the modeled car wheelbase and the true wheelbase. By ranging the model error from  0 to 100\%, we found that although \algName{} still has better performance than other baselines when the error is small ($\leq\!60\%$), \algName{} is generally more sensitive to model error due to its heavy reliance on forward prediction. We also observe similar trends when varying the amount of external disturbances from 0X to 10X.

\heading{Prediction Horizon (Figure~\ref{fig:sensitivity-analysis} lower right))} As an ablation study, we change the prediction horizon $H$ used by \algName{} from 1 to 30 (default value is 20). We see that a very short horizon negatively impacts \algName{}'s performance, suggesting that the local planning step plays an important role, but a horizon longer than 10 (corresponding to 0.5s) does not further improve the performance.

\begin{figure}
    \centering
    \includegraphics[width=0.7\linewidth]{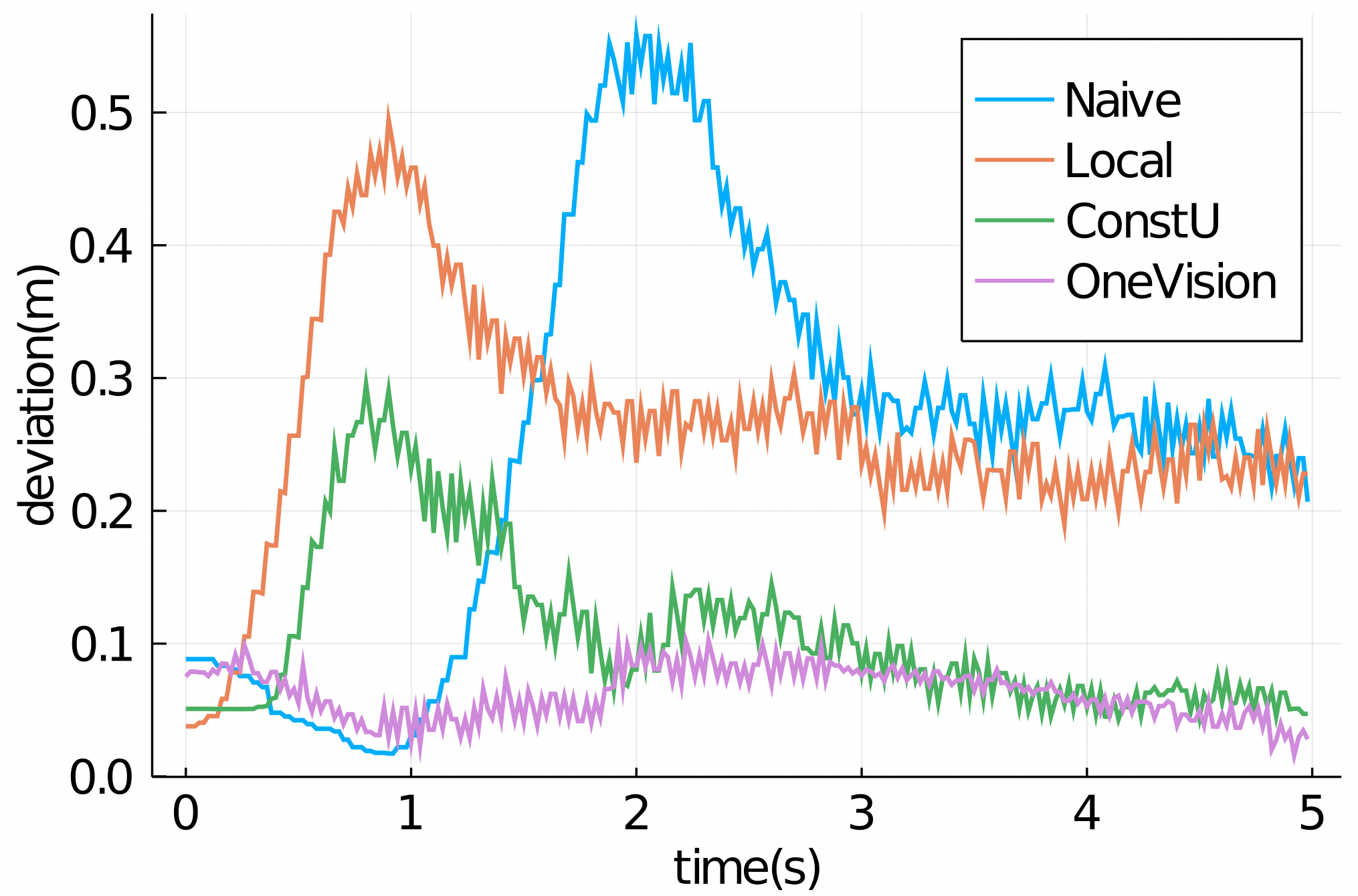}
    \caption{
        \label{fig:deviation}
        Deviation over time in real-world experiment. }
    \vspace{-0.5cm}
\end{figure}

\subsection{Real-World Experiments}

We implemented two real-world versions of the simulation tasks on the UT AUTOmata, a fleet of scale 1/10 autonomous cars. All sensing and computation is performed on-board---the cars are equipped with 2D LIDAR for sensing, and an Nvidia Jetson TX2 for computation. Each car runs Episodic non-Markov Localization~\cite{biswas2016episodic} using observations from the LIDAR to estimate their pose in the world, and communicates to the other cars via WiFi. To ensure that the noise margin of localization is lower than the errors in distributed control stemming from delays, and to account for variability in transmission, the communication queue performs per-message throttling to ensures that all messages have a constant delay of 300ms for all controllers. We run the controllers at 50Hz and use the estimated delay parameters $\stateDelay=$ 40ms and $\actDelay=$ 80ms.  We summarize our results below and highlight the major differences from the simulation tasks.

\heading{1D Leader with Obstacle (quantitative study)} 
We modify our 2D reference point tracking controller to fix the leader's reference point to be always on a straight line. We also use 3 cars instead of 2 and maintain a triangular formation.
For each controller framework, we run the experiment 5 times and report the average deviation below.

\begin{center}
    \small
    \csvautobooktabularcenter{"data/real_deviation.csv"}
\end{center}

We also plot how deviation changes over time in Figure~\ref{fig:deviation}, which matches the expected trends and  shows \algName{}'s superior performance compared to other baselines.

\heading{2D Formation Switching (qualitative study)} We set up the formation so that the leader was tele-operated by one of the authors, while the other two cars followed the leader in formation. Upon command from the leader, the followers had to switch formations while performed obstacle avoidance. We observed robust performance and were able to successfully finish the task using \algName{}. The followers were able to quickly converge back to their target formation in response to variation in the leader's action. We recorded an example execution as part of our supplementary video.

\bibliography{citations} 
\bibliographystyle{abbrv}

% \vfill
% \pagebreak
% \section{Sensitivity Analysis Full Results}
% \begin{sidewaysfigure*}
%     \centering
%     \includegraphics[width=1.1\hsize]{FullResults-Loss.jpg}
%     \caption{Sensitivity Analysis of Log Loss}
%     \label{fig:full_result-loss}
% \end{sidewaysfigure*}

% \begin{sidewaysfigure*}
%     \centering
%     \includegraphics[width=1.1\hsize]{FullResults-Metrics.jpg}
%     \caption{Sensitivity Analysis of Task-Specific Performance Metrics}
%     \label{fig:full_result-metrics}
% \end{sidewaysfigure*}

\end{document}